\PassOptionsToPackage{dvipsnames}{xcolor}
\documentclass[a4paper,12pt]{article}

\RequirePackage{fix-cm}

\usepackage{algorithm}
\usepackage{algorithmicx}
\usepackage{algpseudocode}
\usepackage{graphicx}
\usepackage{xspace}
\usepackage{tikz}
\usepackage{nicefrac}
\usepackage{pgfplots}
\usepackage{amsmath,amssymb,amsthm,amsfonts,mathtools}
\usepackage[inline]{enumitem}
\usepackage[unicode=true,hypertexnames=false]{hyperref}
\usepackage{url}

\usepackage{etoolbox}

\usepackage{fullpage}
\usepackage{authblk}

\pgfplotsset{compat=newest}

\newtheorem{theorem}{Theorem}
\newtheorem{lemma}{Lemma}
\newtheorem{corollary}{Corollary}

\newtheorem{observation}{Observation}

\newcommand{\R}{\mathbb{R}}
\newcommand{\N}{\mathbb{N}}
\DeclarePairedDelimiter\floor{\lfloor}{\rfloor}
\DeclarePairedDelimiter\ceil{\lceil}{\rceil}

\newcommand{\rls}{\textsc{RLS}\xspace}
\newcommand{\ea}{${(1 + 1)}$~EA\xspace}
\newcommand{\eab}{\emph{fast} ${(1 + 1)~\text{EA}_\beta}$\xspace}
\newcommand{\oea}{\ea}
\newcommand{\rea}{${(1 + 1)}$~EA$_{>0}$\xspace}
\newcommand{\mea}{${(1 + 1)}$~EA$_{0 \to 1}$\xspace}
\newcommand{\meatwo}{${(1 + 1)}$~EA$_{0 \to 2}$\xspace}
\newcommand{\muea}{${(\mu + 1)}$~EA\xspace}
\newcommand{\olea}{${(1 + \lambda)}$~EA\xspace}

\newcommand{\om}{\textsc{OneMax}\xspace}
\newcommand{\lo}{\textsc{LeadingOnes}\xspace}

\newcommand{\leadingones}{\lo}
\newcommand{\bv}{\textsc{BinVal}\xspace}

\newcommand{\onemaxplotwidthleft}{0.7\linewidth}

\newcommand{\onemaxplotwidthright}{0.32\linewidth}
\newcommand{\onemaxplotheight}{0.25\textheight}

\pgfplotscreateplotcyclelist{onemaxplotcycle}{%
blue,mark=*,mark size=1\\
red!80!black,mark=square*,mark size=1\\
cyan!80!black,mark=*,mark size=1\\
violet!80!black,mark=square*,mark size=1\\
orange,mark=diamond*\\
}

\newlist{enumeratepar}{enumerate}{1}
\setlist[enumeratepar]{label=\textbf{(\arabic*)}, nosep, leftmargin=0pt, listparindent=\parindent, itemindent=2.8em}

\begin{document}


\title{Fixed-Target Runtime Analysis}
\author[1]{Maxim Buzdalov}\author[2]{Benjamin Doerr}\author[3]{Carola Doerr}\author[1]{Dmitry Vinokurov}\affil[1]{ITMO University, Saint Petersburg, Russia}\affil[2]{Laboratoire d'Informatique (LIX), CNRS, \'Ecole Polytechnique, Institut~Polytechnique de Paris, Palaiseau, France}\affil[3]{Sorbonne Universit{\'e}, CNRS, LIP6, Paris, France}\maketitle

\begin{abstract}
Runtime analysis aims at contributing to our understanding of evolutionary algorithms through mathematical analyses of their runtimes. In the context of discrete optimization problems, runtime analysis classically studies the time needed to find an optimal solution. However, both from a practical and from a theoretical viewpoint, more fine-grained performance measures are needed to gain a more detailed understanding of the main working principles and their resulting performance implications. Two complementary approaches have been suggested: fixed-budget analyses and fixed-target analyses.

In this work, we conduct an in-depth study on the advantages and the limitations of fixed-target analyses. We show that, different from fixed-budget analyses, many classical methods from the runtime analysis of discrete evolutionary algorithms yield fixed-target results without greater effort. We use this to conduct a number of new fixed-target analyses.
However, we also point out examples where an extension of existing runtime results to fixed-target results is highly non-trivial.

\end{abstract}


\section{Introduction}

The classic performance measure in the theory of evolutionary computation~\cite{Jansen13,DoerrN20} is optimization time, that is, the number of fitness evaluations that an algorithm uses to find an optimal solution for a given optimization problem. Often only expected optimization times are analyzed and reported, either for reasons of simplicity or because some analysis methods like certain drift theorems~\cite{Lengler20bookchapter} only yield such bounds.

Some works give more detailed information, e.g., the expectation together with a tail estimate~\cite{Kotzing16,Witt14,DoerrG10datb}. In some situations, only runtime bounds that hold with some or with high probability are given, either because these are easier to prove or more meaningful (see, e.g., the discussion in~\cite{Doerr19gecco} on such statements for estimation-of-distribution algorithms), or because the expectation is unknown~\cite{DoerrL17OM} or infinite. The use of the notion of stochastic domination~\cite{Doerr19tcs} is another way to give more detailed information on runtimes of algorithms.

Nevertheless, all these approaches reduce the whole optimization process to a single point in time: the moment in which an optimal solution is found. For various reasons, more detailed information on the whole process is also desirable, including the following:

\begin{enumeratepar}
\item Evolutionary algorithms, different from most classic algorithms, are so-called \emph{anytime algorithms}. This means that they can be interrupted at essentially any point of time and they still provide some valid solution (possibly of a low quality). The optimization time as the only performance measure gives no information on how good an algorithm is as an anytime algorithm. Such information, however, is of great interest in practice. It can be used, for instance, if one does not know in advance how much time can be allocated to the execution of an algorithm, or when it is important to report whenever a certain milestone (e.g., quality threshold) has been reached.
\item When several iterative optimization heuristics are available to solve a problem, one can try to start the optimization with one heuristic and then, at a suitable time, switch to another one which becomes more powerful at that time. To decide which heuristic to use up to a certain point of time or solution quality, more detailed information than the optimization time is needed.
\end{enumeratepar}

We note that the importance of reporting \emph{runtime profiles} instead of only optimization times has for a long time been recognized in algorithm benchmarking~\cite{cocoplat,iohprofiler}. These fine-grained performance analyses have helped to advance our understanding of evolutionary computation methods, and have contributed significantly to algorithm development. It is therefore surprising that such more fine-grained notions play only a marginal role in the runtime analysis literature. The following two notions have been used in the runtime analysis community.

\begin{itemize}
\item \emph{Fixed-budget analyses:} For a fixed number (``budget'') of fitness evaluations, one studies the (usually expected) quality of the best solution found within this budget.
\item \emph{Fixed-target analyses:} For a fixed quality threshold, one studies the (usually expected) time (often measured in terms of function evaluations) needed to find a solution of at least this quality.
\end{itemize}
	
The main goal of this work is a systematic fixed-target runtime analysis. We provide, in particular,
a comparison of different more fine-grained performance measures (Section~\ref{sec:FBFT}),  
a survey of the existing results (Section~\ref{sec:existing}),  
an analysis how the existing methods to prove runtime analysis results can be used to also give fixed-target results (Sections~\ref{sec:fitnesslevel} and~\ref{sec:drift}) together with several applications of these methods, some to reprove existing results, others to prove new fixed-target results.
The main insight here is that fixed-target results often come almost for free when one can prove a result on the optimization time, so it is a waste to not report them explicitly.  
However, in Section~\ref{sec:difficulties} we also point out situations in which the runtime is well understood, but the derivation of fixed-target results appears very difficult. 

The preliminary version of this work has been published in proceedings of the GECCO conference~\cite{buzdalovDDV-gecco20-fixed-target}.
In this present version, apart for implementing properly the pieces that have been shortened or omitted during to the conference page limit
and extending some of the results to evolutionary algorithms with \emph{fast} mutation operators,
we have introduced new versions of variable and multiplicative drift theorems that explicitly use the expected potential at the moment of stopping
and hence yield better bounds in the fixed-target settings.


\section{Fine-Grained Runtime Analysis: Fixed-Budget and Fixed-Target Analyses}\label{sec:FBFT}

Among the notions other than the time required to find an optimum, the first notion to become the object of rigorous mathematical analysis is \emph{fixed-budget analysis}~\cite{JansenZ14}. Fixed-budget analysis asks, given a computational budget $b \in \N$, for the expected fitness of the best solution seen within $b$ fitness evaluations. In the first paper devoted to this topic (extending results presented at GECCO 2012), Jansen and Zarges~\cite{JansenZ14} investigated the fixed-budget behavior of two simple algorithms, randomized local search (RLS) and the $(1+1)$ evolutionary algorithm (the \ea), on a range of frequently analyzed example problems. For these two elitist algorithms, fixed budget analysis amounts to computing or estimating $f(x_b)$, where $f$ is the objective function and $x_b$ is the $b$-th search point generated by the algorithm. Jansen and Zarges considered small budgets, that is, budgets $b$ below the expected optimization time, and argued that instead of larger budgets, one should rather regard the probability to find the optimum within the budget.

Jansen and Zarges~\cite{JansenZ14} obtained rather simple expressions for the fixed-budget fitness obtained by RLS, but those for the \ea were quite complicated. In~\cite{hypermut}, the same authors evaluated artificial immune systems (AIS). In terms of the classic optimization time, AIS are typically worse than evolutionary algorithms. Interestingly, in the fixed-budget perspective with relatively small budgets, AIS were proven to outperform evolutionary algorithms, confirming our claim that fine-grained runtime results can lead to insights that cannot be (easily) obtained from studying optimization times only.

These first results were achieved using proof techniques highly tailored to the considered algorithms and problems. Given the innocent-looking definition of fixed-budget fitness, the proofs were quite technical even for simple settings like RLS optimizing \leadingones. The analyses were even more complicated for the \oea and many analyses could not cover the whole range of budgets (e.g., for \leadingones, only budgets below $0.5 n^2$ were covered, whereas the (strongly concentrated) optimization time is around $0.86 n^2$, see~\cite{BottcherDN10}).

In~\cite{DoerrJWZ13}, a first general approach to proving fixed-budget results was presented. Interestingly, it works by estimating fixed-target runtimes and then using strong concentration results to translate the fixed-target result into a fixed-budget result. This might be the first work that explicitly defines the fixed-target runtime, without however using this name. The paper~\cite{LenglerS15} also uses fixed-target runtimes (called \emph{approximation times}, see~\cite[Corollary~3]{LenglerS15}) for a fixed-budget analysis, but most of the results in the paper are achieved by employing drift arguments. An explicit collection of drift theorems designed for fixed-budget analyses, along with an application to derive fixed-budget results for the well-studied \leadingones problem, can be found in~\cite{koetzing-witt-fixed-budget-drift}.  

The first fixed-budget analysis for a combinatorial optimization problem was conducted in~\cite{nallaperuma-fitness-gains-tsp}. Subsequently, several papers more concerned with classical optimization time also formulated their results in the fixed-budget perspective, among them~\cite{DoerrDY16,DoerrDY20,DoerrDY16PPSN}. 

A similar notion of fine-grained runtime analysis, called the \emph{unlimited budget analysis}~\cite{unlimited-budget-analysis-gecco19}, was recently proposed. It can be seen as either a complement to the fixed-budget analysis (as its primary goal is to measure how close an algorithm gets to the optimum of the problem in a rather large number of steps) or as an extension of fixed-budget analysis which goes beyond using small budgets only.

For the second main fine-grained runtime notion, \emph{fixed-target analysis}, due to it being a direct extension of the optimization time, it is harder to attribute a birthplace. As we argue also in Section~\ref{sec:fitnesslevel}, the fitness level method is intimately related to the fixed-target view. As such, many classic papers can be seen as fixed-target works, which is particularly true for papers where the fitness level method is not used as a black box, but one explicitly splits a runtime analysis into the time to reach a particular fitness and the another time to find the optimum, as done, e.g., in~\cite{Witt06}. The first, to the best of our knowledge, explicit definition of the fixed-target runtime in a runtime analysis paper can be found in the above-mentioned fixed-budget work~\cite[Section~3]{DoerrJWZ13}. There, for a given algorithm $A$, a given objective function $f$ (to be maximized), and a given target $k \in \R$, the fixed-target runtime $T_{A,f}(k)$ is defined as the number of fitness evaluations after which a search point of fitness at least $k$ is found. Since this notion was merely used as a tool in a proof, the name \emph{fixed-target runtime} was not used yet. The paper~\cite{practice-aware} argued that fixed-target results, coined \emph{runtime profiles} in~\cite{practice-aware}, should be made explicit, to provide more information to practitioners. The name \emph{fixed-target analysis} was, in the context of runtime analysis, first used in the GECCO 2019 student workshop paper~\cite{fixed-target-gecco19}, the only other work putting fixed-target analysis into its center.
It is also worth noting that certain papers combine theoretical results for optimization time and experimental fixed-target results, \cite{doerrYRWB-profiling-one-plus-lambda} being a good example, 
which suggests that there is a demand for a method for an easy translation of the results between these two areas.

In summary, we see that there are generally not too many runtime results that give additional information on how the process progresses over time. Since fixed-budget analysis, as a topic on its own, was introduced earlier into the runtime analysis community, there are more results on fixed-budget analysis. At the same time, by looking over all fixed-budget and fixed-target results, it appears that the fixed-budget results tend to be harder to obtain.

From the viewpoint of designing dynamic algorithms, that is, algorithms that change parameter values or sub-heuristics during the runtime, it appears to us that fixed-budget results are more useful for time-dependent dynamic choices, whereas fixed-target results aid the design of fitness-dependent schemes. If algorithm $A$ with a budget of $b$ computes better solutions than algorithm $B$ with the same budget, then in a time-dependent scheme one would rather run algorithm $A$ for the first $b$ iterations than $B$. However, if the runtime to the fitness target $x$ of algorithm $A$ is lower than that of $B$, then a fitness-dependent scheme would use rather $A$ than $B$ to reach a fitness of at least $x$.

Since we do not see that the increased difficulty of obtaining fixed-budget results is compensated by being significantly more informative or easier to interpret and since we currently see more fitness-dependent algorithm designs (e.g.,~\cite{BottcherDN10,DoerrDE15,DoerrDY16,LehreS20} than time-dependent ones (where, in fact, we are only aware of a very rough proof-of-concept evolutionary algorithm in the early work~\cite{JansenW06}), we advocate in this work to rather focus on fixed-target results. We support this view by further elaborating how the existing analyses methods for the optimization time, in particular, the fitness-level methods and drift, can easily be adapted to also give fixed-target results.


\section{Preliminaries}\label{sec:preliminaries}

Throughout the paper we use the notation $[a..b]$ to denote a set of integer numbers $\{a, a+1, \ldots, b-1, b\}$,
and we denote the set $[1..n]$ as $[n]$. We write $H_n$ for the $n$-th harmonic number, that is, $H_n = \sum_{i=1}^n 1/i$,
and $H_0 = 0$ for convenience. Finally, we use the shorthand $[\mathcal{E}]$ to denote the function that returns 1 if event $\mathcal{E}$ holds true, and which returns 0 otherwise. 

We consider simple algorithms, such as the \ea, the \muea, and the \olea, which solve optimization problems on bit strings of length $n$.
Due to the increased interest in mutation operators that do not produce offspring identical to the parent~\cite{practice-aware},
and to mutation operators different to standard bit mutation, such as the \emph{fast} mutation operators sampling from heavy-tailed distributions~\cite{DoerrLMN17},
we consider them in a generalized form, which samples the number of bits to flip from some distribution $\mathcal{M}$.
We also use a distribution over search points $\mathcal{D}$ during initialization. A default choice for $\mathcal{D}$ is to sample every
bit string with equal probability, however, we consider also initialization with the search point having smallest possible fitness value.
These algorithms are presented in the most general form in Algorithm~\ref{algo:mpl} as a $(\mu+\lambda)$~EA
parameterized by $\mathcal{M}$ and $\mathcal{D}$.

\begin{algorithm}[!t]
\caption{The $(\mu+\lambda)$~EA to maximize $f: \{0,1\}^n \to \R$}\label{algo:mpl}
\begin{algorithmic}
    \Require mutation strength distribution $\mathcal{M}$, initialization distribution $\mathcal{D}$
    \For{$i \in [\mu]$}
        \State $x_i \gets \text{sample from } \mathcal{D}$ 
				\State query $f(x_i)$
    \EndFor
    \State $X \gets \{x_1, \ldots, x_{\mu} \}$
    \While{\textbf{true}}
        \For{$i \in [\lambda]$}
            \State $j \gets \text{sample uniformly from } [\mu]$
			\State $\ell \gets \text{sample from } \mathcal{M}$ 
            \State $y_i \gets \text{flip } \ell \text{ pairwise different, uniformly chosen bits in } x_j$ 
						\State query $f(y_i)$
        \EndFor
        \State $Y \gets \{ y_1, \ldots, y_{\lambda}\}$
        \State $X \gets \mu \text{ best solutions from } X \cup Y$, breaking ties randomly,
        \State \phantom{$X \gets$}~preferring offspring in the case of ties
    \EndWhile
\end{algorithmic}
\end{algorithm}

We consider the following distributions of $\mathcal{M}$ for the \ea:
\begin{itemize}
    \item randomized local search, or RLS: $\mathcal{M} = 1$;
    \item the \ea with standard bit mutation: $\mathcal{M} = B(n,p)$, where $B(n,p)$ is the binomial distribution;
    \item the \mea using the \emph{shift mutation strategy}: $\mathcal{M} = \max\{1, B(n,p)\}$;
    \item the \rea using the \emph{resampling mutation strategy}:\\$\mathcal{M} = [x \sim B(n,p) \mid x > 0]$;
    \item the \emph{fast} \ea with parameter $\beta > 1$: $\mathcal{M} = [x \sim B(n, a/n) \mid a \sim \mathcal{H}_{\beta}]$, and $\mathcal{H}_{\beta}$ is a distribution defined as follows:
    \begin{equation*}
        \Pr[x = a \mid x \sim \mathcal{H}_{\beta}] = a^{-\beta} \cdot (C_{n/2}^{\beta})^{-1},
    \end{equation*}
    where $C_{n/2}^{\beta} = \sum_{i=1}^{n/2} i^{-\beta}$ is the normalization constant.
\end{itemize}

Note that, in fact, we could have also considered the \emph{fast} \ea with the shift or replacement mutation strategy,
as well as the version of the \emph{fast} \ea that directly samples the number of bits to flip from $\mathcal{H}_{\beta}$ without
using the binomial distribution as a proxy. However, since this paper would not benefit from repetitive analyses of very similar
algorithms, we restrict ourselves only to the canonical \emph{fast} \ea.

Sometimes we are only interested in the probability $q$ of flipping a particular bit while not flipping any other bit. 
For problem size $n$ and mutation strength $p$, the values of $q$ for the algorithms above are
\begin{itemize}
    \item RLS: $q = 1/n$;
    \item \ea: $q = p(1-p)^{n-1}$;
    \item \mea: $q = p(1-p)^{n-1} + \frac{(1-p)^n}{n}$;
    \item \rea: $q = \frac{p(1-p)^{n-1}}{1-(1-p)^n}$;
    \item \eab: $q = \frac{1}{n} \cdot (C_{n/2}^{\beta})^{-1} \cdot \gamma(n,\beta) = \frac{1}{n} \cdot (C_{n/2}^{\beta})^{-1} \cdot \Theta(1)$,
          as the factor  $\gamma(n, \beta) = \sum_{i=1}^{n/2} i^{1-\beta} (1-\frac{i}{n})^{n-1}$
          is between $e^{-1}$ and $\frac{2\pi^2}{3}$ by~\cite[Lemma 5]{DoerrLMN17}.
\end{itemize}

We consider the following classical problems on bit strings:
\begin{align*}
    \om(x) &\mapsto \sum\nolimits_{i=1}^n x_i \\
    \lo(x) &\mapsto \sum\nolimits_{i=1}^n \prod\nolimits_{j=1}^i x_i \\
    \bv(x) &\mapsto \sum\nolimits_{i=1}^n 2^{i-1} x_i.
\end{align*}

We also consider the minimum spanning tree (MST) problem. 
Given a connected undirected graph $G$ with positive weights on each edge, the MST problem asks to find a minimum spanning tree of it, that is,
a subgraph that connects all vertices of $G$ and that has the minimum possible weight. 
This problem was adapted to bit strings as in~\cite{NeumannW07} as follows: each bit corresponds to one edge of $G$,
and the bit value of 1 means that the edge is included in the subgraph.


\section{Overview of Known Fixed-Target Results}\label{sec:existing}

In this section, we comment on the known fixed-target results available in recent papers.
We cover both the results and the techniques which have been used and possibly modified to achieve these results.
Where necessary, we estimate the precision of these results by comparing them with the actual expected hitting times of the corresponding algorithms.
In the whole section, $n$~is the problem size and $k$~is the target fitness.

\subsection{\lo}

\lo was the first problem for which certain fixed-target results were derived.
The paper~\cite{Witt06} studied upper and lower bounds on the runtime of the \muea on several benchmark problems.
For \lo, the runtime was proven in~\cite[Theorem~1]{Witt06} using a technique similar to fitness levels, where the state space
also involved the number of the best individuals in the population. Then, \cite[Corollary~1]{Witt06} bounded the expectation
of the time needed to reach a state of at least $k$ leading ones from above by $\mu + 3ek \cdot\max\{\mu \ln en, n\}$, which is a fixed-target result.
In the framework of that paper, this result appeared to be useful in a subsequent part,
where the necessity of having population size $\mu > 1$ was discussed by constructing and analyzing an artificial problem
involving both \om and \lo.

B{\"o}ttcher et al.~\cite{BottcherDN10} proved exact expected runtimes for the \ea on \lo, which was an important cornerstone in the studies of \lo.
In this paper it is proved for the \ea that the expected time to leave a state with the fitness of $i$
and the mutation probability $p$ is $E(A_i) = 1/((1-p)^ip)$. However, the next fitness value may
be greater than $i+1$. The problem-dependent observation, which simplifies the analysis,
is that the bits following the $i$-th bit have a probability of exactly $1/2$ to equal 1.
This yields the expected optimization time to be exactly $((1-p)^{1-n}-(1-p))/(2p^2)$ for a fixed $p$.

This result was reused in~\cite[Section~4]{DoerrJWZ13} to prove that
the expected time that the \oea with mutation probability $p \in (0,1)$ needs to hit target $k \in [n]$ equals  
\begin{equation}
\frac{(1-p)^{1-k}-(1-p)}{2p^2}.\label{eq:lo:ea:fix}
\end{equation}
Technically, the result in~\cite{DoerrJWZ13} is given only for $p = 1/n$,
partially because they used this mutation probability throughout the whole paper.
However, their proof does not depend on the value of $p$ and hence extends to the general case.

Theorem~11 in~\cite{practice-aware} extends this result to similar algorithms: the \rea, the \mea, and RLS.
While only upper bounds are claimed in that work, it is clear from the above that these bounds are exact.
As the results of our techniques are identical, we direct the reader to our Theorem~\ref{result:lo:exact}.

Finally, in~\cite[Corollary 4]{Doerr19tcs} an even stronger result is proven about the exact \emph{distribution} of
the time the \ea needs to hit a certain target. The result is expressed in terms of random variables $X_i$ for initial bit values
and the probabilities $q_i$ to generate a point with a strictly better fitness from a point with fitness $i$,
and reads as $\sum_{i=0}^{k-1} X_i \cdot \text{Geom}(q_i)$, where Geom is the geometric distribution.

The summary of the results for \lo is given in Table~\ref{table:lo:summary}.

\begin{table}[!ht]
\caption{Summary of the fixed-target results available in literature for \lo}\label{table:lo:summary}
\begingroup\centering
\begin{tabular}{lllc}
\hline Algorithm & Result & Constraints & Reference \\
\hline
\muea & $\le \mu + 3ek \cdot \max{\{\mu\ln(en),n\}}$                                       & $p = 1/n$,            & \cite[Cor.~1]{Witt06} \\
      &                                                                                    & $\mu=\text{poly}(n)$  & \\
\ea   & $=\frac{n^2 - n}{2} \cdot \left(\left(1+\frac{1}{n-1}\right)^k - 1\right)$         & $p = 1/n$             & \cite[Sec.~4]{DoerrJWZ13} \\[1.5ex]
\rls  & $=^{(1)} kn/2$                                                                     & --                    & \cite[Th.~11]{practice-aware} \\[1.5ex]
\ea   & $=^{(1)} \frac{(1-p)^{1-k} - (1-p)}{2p^2}$                                         & --                    & \cite[Th.~11]{practice-aware} \\[1.5ex]
\mea  & $=^{(1)} \frac{1}{2}\sum_{i=n-k+1}^{n}\frac{1}{p(1-p)^{n-i} + \frac{1}{n}(1-p)^n}$ & --                    & \cite[Th.~11]{practice-aware} \\[1.5ex]
\rea  & $=^{(1)} \frac{1 - (1-p)^n}{2p^2} ((1-p)^{1-k}-(1-p))$                             & --                    & \cite[Th.~11]{practice-aware} \\
\hline
\end{tabular}
\par\endgroup
$^{(1)}$Note: \cite[Th.~11]{practice-aware} states only the upper bound. However, using the same arguments as in~\cite[Sec.~4]{DoerrJWZ13}, it is not difficult to verify that these bounds are exact.
\end{table}

\subsection{\om}

To the best of our knowledge, explicitly formulated fixed-target results regarding \om exist so far only for the \ea and similar algorithms.
What is more, due to the fact that the \ea shows different expected progress for different fitness values on the same problem,
for which it is hard to find sharp bounds, the available upper and lower fixed-target bounds for \om are generally less precise than those for \lo.
For this reason, we augment the existing theoretical results with the actual runtime profiles of the \ea,
which are given in Figure~\ref{om:profile} for three problem sizes $n = 10^3, 10^4, 10^5$.

\begin{figure}[!t]\centering
\includegraphics[scale=1.22]{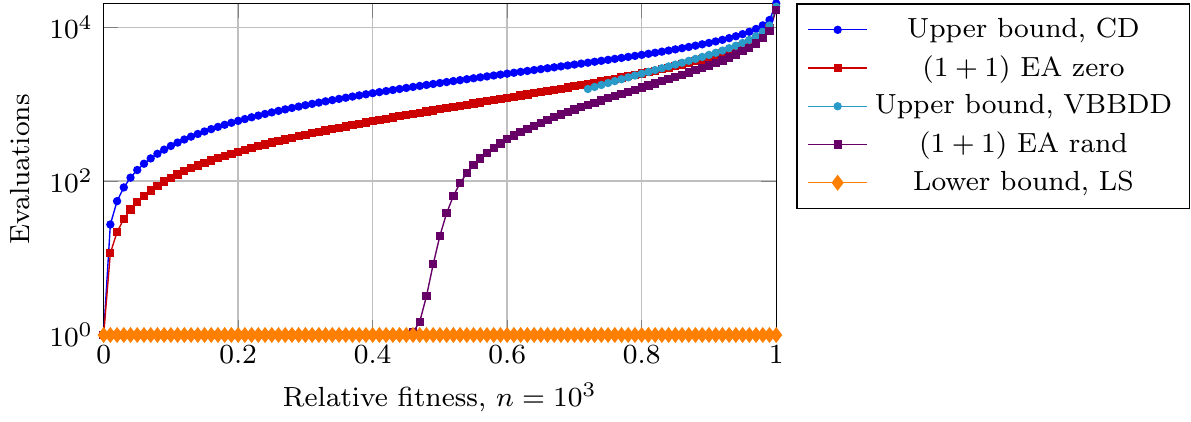}
\par
\includegraphics[scale=1.22]{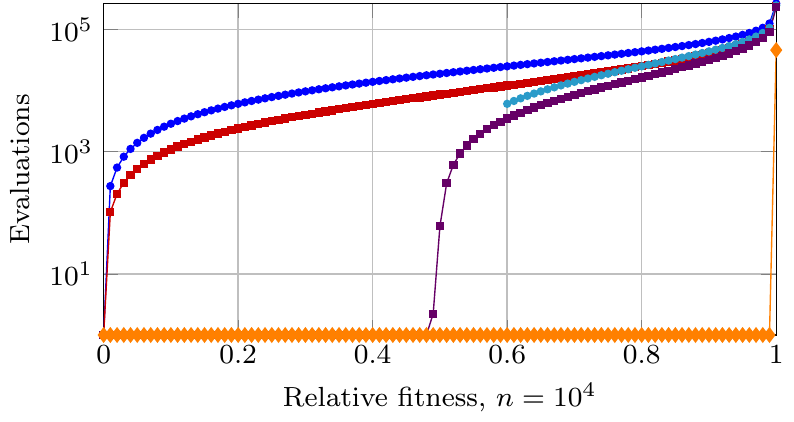}%
\includegraphics[scale=1.22]{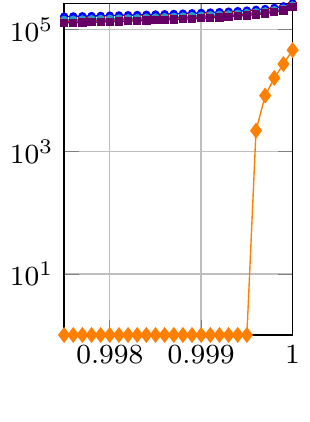}
\par
\includegraphics[scale=1.22]{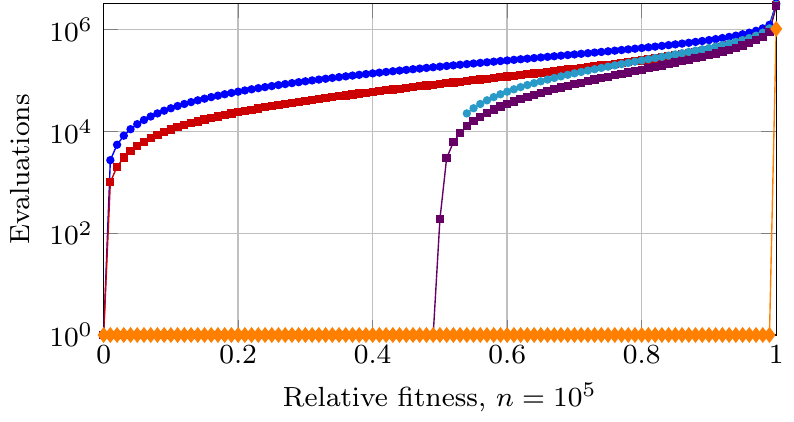}%
\includegraphics[scale=1.22]{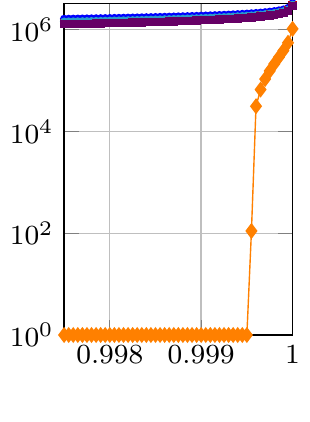}
\par
\caption{The runtime profile of the \ea on \om and the applicable bounds. Plots in the left part of the figure show the general picture for three problem sizes $n \in \{10^3, 10^4, 10^5\}$,
whereas plots in the right part highlight the regions close to the optimum. 
The plotted quantity is the expected number of fitness evaluations required to reach the target fitness value against that (relative) fitness value. 
The runtime profiles for both cases regarding \ea were computed exactly using dynamic programming following the ideas of~\cite{buskulic-doerr}.
For the upper and lower bounds, the respective expressions were evaluated and plotted, disregarding the $(1 \pm o(1))$ factors.
The notation in the legend is as follows: 
``\ea zero'' is the \ea starting at the all-zero bit string,
``\ea rand'' is the \ea starting at a randomly generated bit string,
``Upper bound, CD'' is the bound from~\cite[Theorem~10]{practice-aware},
``Upper bound, VBBDD'' is the bound from~\cite[Theorem~3.1]{fixed-target-gecco19} applicable to ``\ea rand'',
and finally ``Lower bound, LS'' is the bound from~\cite[Corollary~3]{LenglerS15}.
}\label{om:profile}
\end{figure}

The first result for \om appeared in~\cite[Corollary~3]{LenglerS15}, again in a fixed-budget setting.
It is a lower bound, which was proven using fitness levels for lower bounds~\cite{Sudholt13}.
This result says that, for the \ea with mutation probability $p=1/n$ optimizing \om, the expected time until hitting a target $k$ is at least
\begin{equation}
en \ln (n/(n-k)) - 2n \ln \ln n - 16n, \label{eq:om:lengler}
\end{equation}
which was used in~\cite{LenglerS15} to prove fixed-budget results.

This bound is aimed at estimating the time needed to reach targets which are very close to the optimum. Since the \ea spends most of its time
in these regions, and because~\cite{LenglerS15} is dedicated to fixed-budget estimations, this choice seems reasonable.
However, from the fixed-target perspective, \eqref{eq:om:lengler} results in a satisfactory lower bound only for a small fraction of possible conditions.
In particular, it turns negative when the target is smaller than $n \cdot (1 - (\ln n)^{-2/e})$,
so, for instance, it cannot be used to estimate the time needed to reach targets of the form $cn$, where $0 < c < 1$ is a constant factor.
The combination of the constant factors also renders it usable only for quite large values of $n$:
for instance, at $n=10^4$ only five values $k \in [9996..10^4]$ produce nonzero values.
For this reason, Figure~\ref{om:profile} contains special subfigures which highlight this particular bound at targets close to the optimum.

In~\cite[Theorem~10]{practice-aware}, the first fixed-target upper bounds were given for the \ea, the \rea, the \mea, and for RLS on \om.
They cover the whole range of targets, as well as arbitrary mutation probabilities.
They are based on simple fitness level arguments, and for simplicity they assume the worst case regarding initialization,
that is, that the algorithm starts at the all-zero bit string. As a result, all these bounds have a form of
of $\frac{1}{q}(H_n - H_{n-k})$, where $H_i$ is the $i$-th harmonic number and $q$ is the lower bound on the probability
of flipping exactly one bit. This bound is exact for RLS and also captures the behavior of other \ea flavors quite well
(see Figure~\ref{om:profile} for the example of the standard \ea).

To quantify how much the upper bound for the \ea on \om from~\cite[Theorem~10]{practice-aware}
differs from the actual times needed for the \ea to hit various targets when starting at the all-zeros bit string,
we reuse the data from Figure~\ref{om:profile} to compute the ratios of the upper bound to the corresponding average
target-hitting time. We present these insights in Table~\ref{table:om:upper:gaps}, from which one can infer, for instance, that
the upper bound from~\cite[Theorem~10]{practice-aware} overestimates the true hitting times at most
by a factor that tends to some constant value.

\begin{table}[!t]
\caption{Insights about the ratio of the upper bound on the fixed-target runtime of the \ea on \om from~\cite[Th.~10]{practice-aware}
         and the actual expected hitting times computed using dynamic programming following the ideas of~\cite{buskulic-doerr}.
         We present the maximum of the ratio across all targets $k \in [n]$, as well as the ranges of targets
         where the ratio is at least some fixed value (values of 1.5 and 2.5 are used).
         Note how the left endpoints of these ranges remain the same as the problem size $n$ grows,
         whereas the right endpoints scale linearly with remarkable precision.
         The maximum ratio slowly grows and apparently approaches a constant, which resembles $e = 2.71\ldots$ quite closely.}
         \label{table:om:upper:gaps}
\centering
\begin{tabular}{cccccc}
\hline $n$ & Maximum & \multicolumn{2}{c}{Ratio $\ge 2.5$} & \multicolumn{2}{c}{Ratio $\ge 1.5$} \\
           & ratio   & Min target & Max target & Min target & Max target \\\hline
$10^3$     & $2.606391$ & 19  & 221    & 2 & 920 \\
$10^4$     & $2.681519$ & 18  & 2320   & 2 & 9205 \\
$10^5$     & $2.706512$ & 18  & 23306  & 2 & 92057 \\\hline
\end{tabular}
\par
\end{table}

While being moderately precise for estimating the fixed-target runtimes of the \ea, and similar algorithms, on \om starting from the all-zeros bit string,
these bounds overestimate the fixed-target results much more when these algorithms start from a randomly generated bit string, which is a more realistic condition.
For random initialization,~\cite[Theorem~3.1]{fixed-target-gecco19} proves a similar upper bound of $\frac{1}{q} (H_{n/2} - H_{n-k}) (1-o(1))$
for $k > n/2 + \sqrt{n}\ln n$ using the same technique and additional arguments to use $H_{n/2}$. Technically, this result was proven only for \rea,
however, the proof holds also for arbitrary probabilities $q$ of flipping exactly one bit. Such a bound gives a moderately good approximation,
as illustrated in Figure~\ref{om:profile} for the classic \ea. The evaluation of the approximation quality, similar to the one shown in Table~\ref{table:om:upper:gaps},
shows that the constant factors in the approximation are better: the maximum ratio for $n=10^5$ is less than $1.8$, and the ratio gets smaller than $1.5$ at smaller $k \approx 0.815 n$.

In~\cite[Theorem 6.5]{Witt13}, lower bounds for the runtimes of the \ea on \om
are proven, assuming arbitrary mutation probabilities $p$ satisfying $p = O(n^{-2/3-\varepsilon})$.
For this, an interval $S = [n\tilde{p}\ln^2 n; 1/(2\tilde{p}^2n \ln n)]$ of distances to the optimum is considered,
where $\tilde{p} = \max\{1/n, p\}$. Multiplicative drift for lower bounds
is applied to this interval to yield the lower bound of
\begin{equation*}
(1 - o(1)) (1-p)^{-n} (1/p) \min\{\ln n, \ln(1/(p^3n^2))\}.
\end{equation*}

This result was used in~\cite{fixed-target-gecco19} to obtain the fixed-target results for the \rea on \om.
The small difference between the \ea and the \rea, from the point of view of the above-mentioned proof,
is only in one of the factors, which depends on $p$ in a slightly different way.
The obtained fixed-target results are piecewise. 
Indeed, \cite[Theorem 6.5]{Witt13} uses the time that the \ea needs to go through an interval of fitness values $S$.
For this reason, for all targets $k$ that do not belong to this interval, $k \ge n - n\tilde{p} \ln^2 n$,
the fixed-target runtime bound does not depend on $k$.
What is more, the bounds contain a factor $1 - o(1)$ which is hard to estimate well, and whose value appears to be large enough for practical values of $n$.
For this reason we do not show the plots for these bounds in Figure~\ref{om:profile}.

The summary of the results for \om is given in Table~\ref{table:om:summary}.

\begin{table}[!t]
\scriptsize\setlength{\tabcolsep}{0.3em}
\caption{Summary of the fixed-target results available in literature for \om}\label{table:om:summary}
\begingroup\centering
\begin{tabular}{lllc}
\hline Algorithm & Result & Constraints & Reference \\
\hline
\ea  & $\ge en \ln (n/(n-k)) - 2n \ln \ln n - 16n$                & $p=1/n$            & \cite[Cor.~3]{LenglerS15} \\[1.8ex]
\rls & $\le n \cdot (H_n - H_{n-k})$                              & --                 & \cite[Th.~10]{practice-aware} \\[1.8ex]
\ea  & $\le \frac{H_n - H_{n-k}}{p(1-p)^{n-1}}$                   & --                 & \cite[Th.~10]{practice-aware} \\[1.8ex]
\mea & $\le \frac{H_n - H_{n-k}}{(1-p)^{n-1} (p+1/n-p/n)}$        & --                 & \cite[Th.~10]{practice-aware} \\[1.8ex]
\rea & $\le \frac{1 - (1-p)^n}{p(1-p)^{n-1}} (H_n - H_{n-k})$     & --                 & \cite[Th.~10]{practice-aware} \\[1.8ex]
\rea & $\le \frac{1 - (1-p)^n}{p(1-p)^{n-1}} (H_{n/2} - H_{n-k}) \cdot (1-o(1))$ & $k > n/2 + \sqrt{n}\ln n$ & \cite[Th.~3.1]{fixed-target-gecco19} \\[1.8ex]
\rea & $\ge (1-o(1)) \frac{1-(1-p)^n}{p(1-p)^{n-1}} \ln \frac{1}{4\tilde{p}^3n^2 \ln^3 n}$
                                                                                    & $p = O(n^{-2/3-\varepsilon})$,  & \cite[Th.~3.4]{fixed-target-gecco19} \\[1.8ex]
     & where $\tilde{p} = \max\{1/n, p\}$                         & $k \ge n - n\tilde{p} \ln^2 n$  & \\[1.8ex]
\rea & $\ge (1-o(1)) \frac{1-(1-p)^n}{p(1-p)^{n-1}} \ln \frac{1}{4\tilde{p}^2n(n-k) \ln n}$
                                                                                    & $p = O(n^{-2/3-\varepsilon})$,  & \cite[Th.~3.4]{fixed-target-gecco19} \\[1.8ex]
     & where $\tilde{p} = \max\{1/n, p\}$                         & $k \le n - n\tilde{p} \ln^2 n$, & \\[1.8ex]
     &                                                            & $n - k = o(1/(2\tilde{p}^2n \ln n))$    & \\
\hline
\end{tabular}
\par\endgroup
\end{table}

\subsection{\bv}

The fixed-target bounds for \bv were proven in~\cite{fixed-target-gecco19} for the \rea.
The methods for proving optimization times for linear functions, such as the ones in~\cite{Witt13}, were found to be insufficient,
so a problem-dependent observation was used. To achieve a target value $k$ such that $2^n - 2^t \le k < 2^n - 2^{t+1}$,
one requires to optimize the $t$ heaviest bits, and it is enough to optimize the $t+1$ heaviest bits.
As a result, reaching the target $k$ is equivalent to solving \bv of size $t + O(1)$ to optimality using a $n/t$ times smaller
mutation rate. The quite complicated bounds from~\cite{Witt13} were adapted to the case of \bv in~\cite[Theorem 4.1]{fixed-target-gecco19}.
The summary of these results is given in Table~\ref{table:bv:summary}.

\begin{table}[!t]
\scriptsize\setlength{\tabcolsep}{0.3em}
\caption{Summary of the fixed-target results for \bv}\label{table:bv:summary}
\begingroup\centering
\begin{tabular}{lllc}\hline
Algorithm & Result & Constraints & Reference \\\hline
\rea & $\ge (1 - o(1)) \frac{1 - (1-p)^n}{p (1-p)^{n^{-}}} \min\left\{\ln n^{-}, \ln \frac{1}{p^3 (n^{-})^2}\right\}$,
                                                                                                            & $p = O(n^{-2/3-\varepsilon})$ & \cite[Th.~4.1]{fixed-target-gecco19} \\
     & $n^{-} = n-\lceil\log_2 (2^n - k)\rceil$                                                             &                            & \\[2ex]
\rea & $\le \frac{pn^{+}\alpha^2(1-p)^{1-n^{+}} + \alpha\left(\ln\frac{1}{p} + (n^{+}-1) \ln (1-p) + 1\right)}{(1-p)^{n^{+}-1} \cdot p(\alpha-1) \cdot (1 - (1-p)^n)^{-1}}$,
                                                                                                            & $p = O(n^{-2/3-\varepsilon})$ & \cite[Th.~4.1]{fixed-target-gecco19} \\
     & $n^{+} = n-\lfloor\log_2 (2^n - k)\rfloor$,                                                          &                            & \\
     & $\alpha$ arbitrary e.g. $\ln \ln n^{+}$                                                              &                            & \\
\hline 
\end{tabular}
\par\endgroup
\end{table}


\section{Fitness Levels}\label{sec:fitnesslevel}

In this section we consider the fitness level theorems in the fixed-target context. 
A key take-away, implicitly known in the community for years, is that the most important theorems of this sort are already suitable to produce
fixed-target results.


\subsection{Fitness Level Theorems}

In the fitness level method (also known as \emph{artificial fitness levels} and \emph{the method of $f$-based partitions}~\cite{Wegener02}),
the state of the algorithm is typically described by the best fitness of an individual in the population.
It is transformed into a \emph{fitness level}, which may aggregate one or more consecutive fitness values.
For the fitness function $f$ (to be maximized) and two search space subsets $A$ and $B$ one writes $A <_f B$ if $f(a) < f(b)$ for all $a \in A$ and $b \in B$.
A partition $\{A_1 \ldots A_m\}$ of the search space, such that $A_i <_f A_{i+1}$ for all $1 \le i < m$ and $A_m$ contains only the optima of the problem, is called an \emph{$f$-based partition}.
If the best individual in a population of an algorithm $\mathcal{A}$ has a fitness $f \in A_i$, then the algorithm is said to be in $A_i$.

Fitness level theorems work best for the algorithms that typically visit the fitness levels one by one in the increasing order.
We summarize below several popular fitness level theorems. We note that a simpler version of the fitness level theorems in~\cite{Sudholt13} was recently given in~\cite{DoerrK21gecco}. Since this work appeared significantly after the original preparation of this manuscript, we do not discuss it any further, but note nevertheless that for future fixed-target analyses it might be easier to use than the theorems from~\cite{Sudholt13}.

\begin{theorem}[Fitness levels for upper bounds; Lemma~1 from~\cite{Wegener02}]
Let $\{A_i\}_{1 \le i \le m}$ be an $f$-based partition, and let $p_i$ be a lower bound for the probability that the elitist algorithm $\mathcal{A}$
samples a search point belonging to $A_{i+1} \cup \ldots \cup A_m$ provided it currently is in~$A_i$.
Then the expected hitting time of $A_m$ is at most\label{fitness-levels-upper}
\begin{equation}
\sum_{i=1}^{m - 1} P[\mathcal{A} \text{ starts in } A_i] \cdot \sum_{j=i}^{m-1} \frac{1}{p_i} \le \sum_{j=1}^{m-1} \frac{1}{p_i}.
\end{equation}
\end{theorem}

\begin{theorem}[Fitness levels for lower bounds; Theorem~3 from~\cite{Sudholt13}]
Let $\{A_i\}_{1 \le i \le m}$ be a partition of the search space. Let the probability for the elitist algorithm $\mathcal{A}$ to transfer in one step from $A_i$
to $A_j$, $i < j$, to be at most $u_i \cdot \gamma_{i,j}$, and let $\sum_{j=i+1}^m \gamma_{i,j} = 1$. Assume there is some $\chi \in [0,1]$ 
such that $\gamma_{i,j} \ge \chi \sum_{k=j}^m \gamma_{i,k}$ for all $1 \le i < j \le m$.
Then the expected hitting time of $A_m$ is at least\label{fitness-levels-lower}
\begin{equation}
\sum_{i=1}^{m - 1} P[\mathcal{A} \text{ starts in } A_i] \cdot \left(\frac{1}{u_i} + \chi \sum_{j=i+1}^{m-1} \frac{1}{u_i}\right).
\end{equation}
\end{theorem}

We also use an upper-bound theorem similar to Theorem~\ref{fitness-levels-lower}.

\begin{theorem}[Refined fitness levels for upper bounds; Theorem~4 from~\cite{Sudholt13}]
Let $\{A_i\}_{1 \le i \le m}$ be a partition of the search space. Let the probability for the elitist algorithm $\mathcal{A}$ to transfer in one step from $A_i$
to $A_j$, $i < j$, to be at least $s_i \cdot \gamma_{i,j}$, and let $\sum_{j=i+1}^m \gamma_{i,j} = 1$. Assume there is some $\chi \in [0,1]$ 
such that $\gamma_{i,j} \le \chi \sum_{k=j}^m \gamma_{i,k}$ for all $1 \le i < j < m$. 
Further, assume that $(1 - \chi) s_i \le s_{i+1}$ for all $1 \le i \le m - 2$. 
Then the expected hitting time of $A_m$ is at most\label{fitness-levels-upper-refined}
\begin{equation}
\sum_{i=1}^{m - 1} P[\mathcal{A} \text{ starts in } A_i] \cdot \left(\frac{1}{s_i} + \chi \sum_{j=i+1}^{m-1} \frac{1}{s_i}\right).
\end{equation}
\end{theorem}

Theorems~\ref{fitness-levels-upper}--\ref{fitness-levels-upper-refined} are applicable only to elitist algorithms.
However, a fitness level theorem was proposed in~\cite{CorusDEL18} that can be applied to non-elitist algorithms as well (see~\cite{DoerrK19} for a sharper version of this result).

\begin{observation}
If one of Theorems~\ref{fitness-levels-upper}--\ref{fitness-levels-upper-refined} is proven for a certain algorithm on a certain problem,
it also holds if a new target fitness level $m'$ is chosen, such that $1 < m' < m$,
and all fitness levels $A_{m'}, A_{m'+1}, \ldots, A_m$ are merged.\label{fitness-levels-observation}
\end{observation}
\begin{proof}
This is essentially the same argument as in~\cite{LenglerS15}.

This modification does not alter the estimations of probabilities to leave fitness levels preceding $m'$:
$p_i$, $u_i$ and $s_i$ for Theorems~\ref{fitness-levels-upper}, \ref{fitness-levels-lower}, and~\ref{fitness-levels-upper-refined}, respectively. 
The only affected locations are the constraints on $\gamma_{i,j}$. Their affected occurrences on the right-hand sides
effectively merge by summing up, e.g., $\gamma'_{i,m'} \gets \sum_{k=m'}^{m} \gamma_{i,k}$. Note that
only those right-hand sides,  which contain the complete sums from $m'$ to $m$, survive after the transformation, and not just their parts.
For the left-hand sides, only those $\gamma_{i,j}$ survive where $j = m'$, as all others
are either unchanged or cease to exist. However, these occurrences are trivial,
since they are limited only by identity inequalities $\gamma_{i,m'} \le \chi \gamma_{i,m'}$ in Theorem~\ref{fitness-levels-lower}
and are not limited by anything in Theorem~\ref{fitness-levels-upper-refined} as their limits are conditioned on $j < m'$. 
\end{proof}

It follows from Observation~\ref{fitness-levels-observation} that it is very easy to obtain fixed-target results
from the existing optimization time results whose proofs use the above-mentioned fitness level theorems.

Note that, technically, Theorems~\ref{fitness-levels-lower}, \ref{fitness-levels-upper-refined}
and the theorem from~\cite{CorusDEL18} do not require the employed partition of the search space to be an $f$-based partition.
Formally speaking, this enables using them in the fixed-target context as is, contrary to the original formulation of fitness levels
for upper bounds, where the last partition must contain optima and nothing else.
However, this alone does not yet allow to reuse the existing optimization time results in order to obtain the fixed-target ones.
Observation~\ref{fitness-levels-observation} fills this gap by specifying the sufficient conditions for this to be possible.

We also note that Theorems~\ref{fitness-levels-lower} and \ref{fitness-levels-upper-refined}
only require the transition graph over the employed partition to be acyclic with regard to the analyzed algorithm,
but does not require that a fitness level $A_m$ is reachable from another fitness level $A_n$ for every $n < m$.
In fact, Observation~\ref{fitness-levels-observation} may be extended without much effort to allow merging of fitness levels whose indices do not form
a contiguous sequence, provided that one cannot reach a non-merged fitness level from any of the merged levels.
This feature might have applications in analysis of certain local search algorithms.


\subsection{Applications}
\subsubsection{Hill Climbers on \lo}

We re-prove here the statements about the fixed-target performance of the algorithms from the \ea family,
which were proven in~\cite[Section~4]{DoerrJWZ13} and stated, but not formally proven, in~\cite[Theorem~11]{practice-aware}.
For this we use Theorems~\ref{fitness-levels-lower} and~\ref{fitness-levels-upper-refined}, similarly to their use in~\cite{Sudholt13}
to prove the results from~\cite{BottcherDN10} with fitness levels alone.

\begin{lemma}
    In the context of Theorems~\ref{fitness-levels-lower} and~\ref{fitness-levels-upper-refined} for \lo,
    assuming the target fitness is $k$, the values $\gamma_{i,j} = 2^{i-j}$ when $j < k$,
    $\gamma_{i,k} = 2^{i-k+1}$, and $\chi = 1/2$ satisfy their requirements. \label{lemma-lo-params}
\end{lemma}
\begin{proof} 
    Recall that for \lo, if $i$ is the fitness of an individual, its first $i+1$ bits are known, and the bits at indices $\{i+2, \ldots, n\}$ are uniformly distributed.
	Thus, conditioned on mutation being applied and the fitness being improved, the probability of transferring to fitness level $j$ is $\gamma_{i,j} = 2^{-j+i+[j=n]}$.
	Due to~\cite[Theorem~5]{Sudholt13}, the constant $\chi$ to be used in the theorems above is equal to $1/2$.
    
	In the fixed-target context, the transition probabilities $\gamma_{i,j}$ where $j < k$ remain the same,
    as the underlying process is not changed, whereas since fitness levels $k$, $k+1$, \ldots, $n$ merge,
    the probability $\gamma_{i,k} = \gamma^{\text{old}}_{i,k} + \gamma^{\text{old}}_{i,k+1} + \ldots$
    results in telescoping the inverse powers of two. Since fitness level merging effectively reduces the problem size from $n$ to $k$
    without affecting anything else, the constant value $\chi = 1/2$ also remains valid. 
\end{proof}

\begin{lemma}
    Assume that $q_i$ is the probability for algorithm $\mathcal{A}$ to flip a given bit while not flipping any of some other $i$~given bits.
    The expected time for $\mathcal{A}$ to reach a target of at least $k$ on \lo of size $n$ is:\label{lemma-lo-sum}
    \begin{equation}
        \sum_{i=0}^{k-1} P[\mathcal{A} \text{ starts with fitness } i] \cdot \left(\frac{1}{q_i} + \frac{1}{2} \sum_{j=i+1}^{k-1} \frac{1}{q_j}\right)
        = \frac{1}{2} \sum_{i=0}^{k-1} \frac{1}{q_i}.
    \end{equation}
\end{lemma}
\begin{proof}
    The left-hand side of the lemma statement follows from Lemma~\ref{lemma-lo-params},
    Theorems~\ref{fitness-levels-lower} and~\ref{fitness-levels-upper-refined},
    and Observation~\ref{fitness-levels-observation}.
    The right-hand side follows by recalling that in \lo,
    it holds for all considered algorithms that $P[\mathcal{A} \text{ starts with fitness } i] = 2^{-i-1}$,
    and by reordering the sums as in~\cite{BottcherDN10}.

    We can also derive this result from~\cite[Corollary 4]{Doerr19tcs}. 
\end{proof}

\begin{theorem}
    The expected fixed-target time to reach a target of at least $k$ when optimizing \lo of size $n$ is exactly \label{result:lo:exact}
    \begin{itemize}
        \item $\frac{kn}{2}$ for RLS;
        \item $\frac{(1-p)^{1-k} - (1-p)}{2p^2}$ for the \ea with mutation probability $p$;
        \item $\frac{(1-p)^{1-k} - (1-p)}{2p^2}(1 - (1-p)^n)$ for the \rea with mutation probability $p$;
        \item $\frac{1}{2} \cdot \sum_{i=0}^{k-1} \frac{1}{p(1-p)^i + \frac{1}{n}(1-p)^n}$ for the \mea with mutation probability $p$; 
        \item $\frac{1}{2} C_n^{\beta} \cdot \sum_{i=0}^{k-1} \frac{1}{\sum_{j=1}^{n/2} j^{-\beta} \frac{j}{n} (1 - \frac{j}{n})^i}$ for the \eab.
    \end{itemize}
\end{theorem}
\begin{proof}
    We use Lemma~\ref{lemma-lo-sum} and note that, for a fitness of $i$:
    \begin{itemize}
        \item for RLS, $q_i = 1/n$;
        \item for the \ea, $q_i = (1-p)^i \cdot p$;
        \item for the \rea, $q_i = (1-p)^i \cdot p \cdot (1-(1-p)^n)$;
        \item for the \mea, $q_i = (1-p)^i \cdot p + \frac{1}{n} (1-p)^n$;
        \item for the \eab, $q_i = (C_n^{\beta})^{-1} \cdot \sum_{j=1}^{n/2} j^{-\beta} \frac{j}{n} (1 - \frac{j}{n})^i$.
    \end{itemize}
    The sums for the \ea and the \rea are simplified as in~\cite{BottcherDN10}. 
\end{proof}

Note that the presented expression for the \eab is complicated, and little can be understood from it about its behavior, and in particular about the dependency on $\beta$.
The simplified upper bound can be obtained by considering only $j=1$, in which case it is greater by a factor of $C_n^{\beta}$ than the corresponding fixed-target time
for the \ea, and this factor is a constant for $1 < \beta < 2$.
Lower bounds, however, require much more work, as, for instance, $q_i$ for $i = \Theta(1)$ is $\Theta(n^{1-\beta})$, so the \eab performs the first improvements much faster
than more conventional versions of the \ea. It is hence quite hard to obtain closed-form sharp bounds, which we leave for possible future work.


\subsubsection{Hill Climbers on \om, Upper Bounds}

We now re-prove the existing results for the \ea variants on \om. Our results
for the case of random initialization of an algorithm are sharper than in~\cite{fixed-target-gecco19},
because we use the following argument about the weighed sums of harmonic numbers.

\begin{lemma}
    The following equality holds:\label{lemma-choose-sum}
    \begin{equation*}
        \sum_{i=0}^n \frac{\binom{n}{i} H_i}{2^n} = H_n - \sum_{k=1}^n \frac{1}{k 2^k} = H_n - \ln 2 + O(2^{-n}) = H_{n/2} - \frac{1-o(1)}{2n}.
    \end{equation*}
\end{lemma}
\begin{proof}
    Proven in~\cite[Sec.~2.5]{DoerrD16} with~\cite[Identity 14]{combinatorial-sums-finite-differences}. 
\end{proof}

Our results are formalized as follows.

\begin{theorem}
    The expected fixed-target time for Algorithm~\ref{algo:mpl} with $\mu = \lambda = 1$,
    whose mutation distribution $\mathcal{M}$ selects a single bit to flip with probability $q$,
    to reach a target of at least $k$ on \om of size $n$ is at most:
    \label{result:om:up}
    \begin{itemize}
        \item $\frac{1}{q} (H_n - H_{n-k})$ when initializing with the worst solution;
        \item $\frac{1}{q} (H_{n/2} - H_{n-k} - \frac{1-o(1)}{2n})$ when initializing randomly, assuming $k \ge n/2 + 2\sqrt{n \ln n}$.
    \end{itemize}
\end{theorem}
\begin{proof}
    Let $s_i$ be the probability for the algorithm to be initialized at fitness $i$.
    Assuming pessimistically that the fitness does not improve when two and more bits are flipped, we apply Theorem~\ref{fitness-levels-upper} to get the following upper bound:
    \begin{equation*}
        \sum_{i=0}^{k-1} s_i \cdot \sum_{j=i}^{k-1} \frac{1}{q (n - j)} =\frac{1}{q} \sum_{i=0}^{k-1} s_i \cdot (H_{n-i} - H_{n-k})
        \le \frac{H_n - H_{n-k}}{q}.
    \end{equation*}

    The pessimistic bound above proves the theorem for the algorithms initialized with the worst solution.
    For the random initialization, we note that the initial search point has a fitness $i$ with the probability of $\binom{n}{i} / 2^n$.
    From the equality above we derive:
    \begin{align*}
        &\sum_{i=0}^{k-1} \frac{\binom{n}{i} (H_{n-i} - H_{n-k})}{q 2^n}\\
        &\le \sum_{i=0}^{n} \frac{\binom{n}{i} (H_{n-i} - H_{n-k})}{q 2^n} + \sum_{i=k}^n \frac{\binom{n}{i} (H_{n-k} - H_{n-i})}{q 2^n}\\
        &= \sum_{i=0}^{n} \frac{\binom{n}{i} H_{n-i}}{q 2^n} - \sum_{i=0}^{n} \frac{\binom{n}{i} H_{n-k}}{q 2^n} + \frac{1}{q}\sum_{i=k}^n \frac{\binom{n}{i} (H_{n-k} - H_{n-i})}{2^n}\\
        &= \frac{1}{q} \left(H_{n/2} - H_{n-k} - \frac{1-o(1)}{2n}\right) + \frac{1}{q} \sum_{i=k}^n \frac{\binom{n}{i} (H_{n-k} - H_{n-i})}{2^n},
    \end{align*}
    where the last transformation uses Lemma~\ref{lemma-choose-sum}. The second addend is $o(1/(qn))$, because
    $H_{n-k} - H_{n-i} = O(\ln n)$ when $i \ge k$, which is further multiplied by $O(n^{-8/3})$, since
    for $k \ge n/2 + 2\sqrt{n \ln n}$ it holds that \[\frac{1}{2^n} \sum_{i=k}^n \binom{n}{i} \le \exp\left(-\frac{16 \frac{\ln n}{n} \frac{n}{2}}{3}\right) = n^{-8/3}\] by the Chernoff bound.
    As a result, the fixed-target upper bound for the random initialization is $\frac{1}{q} (H_{n/2} - H_{n-k} - \frac{1-o(1)}{2n})$. 
\end{proof}


\subsubsection{Hill Climbers on \om, Lower Bounds}

We also apply fitness-levels to the lower bounds to improve the result from~\cite{LenglerS15}.
\begin{theorem}[Lower fixed-target bounds on \ea for \om]\label{onemax-lower-bounds-levels}
	For mutation probability $p \le 1/(\sqrt{n} \log n)$, and assuming that
	$\ell = \ceil{n-\min\{n/\log n, 1/(p^2n\log n)\}}$, the expected time 
	to find a solution with fitness at least $k \ge \ell$ is at least
	\begin{equation*}
	    \left(1 - \frac{18/5}{(1-p)^2 \log n}\right) \frac{H_{n-\ell} - H_{n-k}}{p(1-p)^{n-1}}.
	\end{equation*}
\end{theorem}

\begin{proof}
	We use the proof of \cite[Theorem 9]{Sudholt13} with Observation~\ref{fitness-levels-observation}.
	The source of this formula is located at~\cite[journal page 428, bottom of left column]{Sudholt13}.
	We use this particular formula, as it gives a good enough precision for a wide range of possible targets
	and matches the shape of Theorem~\ref{result:om:up} quite well. 
\end{proof}

We conjecture that a similar result can be derived for other versions of the \ea as well, however,
finding the exact constant factors would require additional investigation into~\cite[Theorem 9]{Sudholt13}. 


\subsubsection{The \texorpdfstring{\muea}{(mu+1) EA} on \om, Upper Bounds}

We now introduce the fixed-target bounds for the \muea using fitness levels.
For this we adapt~\cite[Theorem 2]{Witt06} to support fixed targets different from the optimum,
which makes it no longer possible to use certain transitions from the original proof.

\begin{theorem}
	Let $\mu=poly(n)$, and assume $b=\floor*{n(1-1/\mu)}$. The expected time to reach an individual with fitness at least $k>0$ on \om is: \label{result:mu+1}
	\begin{align*}
		T_k \le \mu &+ \frac{\mu}{(1-p)^n} \left(2k - 1 - (n - k) \ln \frac{n}{n-k+1}\right) \\
		            &+ \frac{\mu}{p(1-p)^{n-1}}\begin{cases}
						\frac{k}{n},                          & k \le b+1,\\
						\frac{b+1}{n}+\frac{1}{\mu}(H_{n-b-1}-H_{n-k})    & \text{otherwise}.
					\end{cases}
	\end{align*}
\end{theorem}

\begin{proof}
    Similarly to~\cite[Theorem 2]{Witt06}, we pessimistically assume that on every fitness level $L$ the \muea creates $R = \min\{\mu, n / (n - L)\}$ replicas of the best individual,
    and then it waits for the fitness improvement. We also assume that the \muea never improves the fitness by more than one.

    If there are $i < R$ best individuals, the probability of creating a replica is $(1 - p)^n i / \mu$, so the expected time until creating $R$ replicas is at most
    \begin{equation*}
		\frac{\mu}{(1-p)^n}\sum_{i=1}^{R-1}\frac{1}{i} \le \frac{\mu}{(1-p)^n}\sum_{i=1}^{\frac{L}{n-L}}\frac{1}{i} 
		\le \frac{\mu}{(1-p)^n}\ln{\frac{en}{n-L}}
    \end{equation*}
    and the total time the \muea spends in creating replicas is
    \begin{align*}
        T_r &\le \sum_{L=0}^{k-1} \frac{\mu}{(1-p)^n}\ln{\frac{en}{n-L}} = \frac{\mu}{(1-p)^n} \left(k \ln en + \sum_{\mathclap{x = n-k+1}}^n \ln \frac{1}{x}\right) \\
         &\le \frac{\mu}{(1-p)^n} (k \ln en + k - 1 + (n-k) \ln (n-k+1) - n\ln n)\\
         &= \frac{\mu}{(1-p)^n} \left(2k - 1 - (n - k) \ln \frac{n}{n-k+1}\right),
    \end{align*}
    where we use the condition $k>0$ and that, for all $1 \le a \le b$,
    \begin{equation*}
        \sum_{x=a}^b \ln \frac{1}{x} \le \ln \frac{1}{a} + \int_a^b \ln \frac{1}{x} dx = \ln \frac{1}{a} + \left[x + x \ln \frac{1}{x}\right]_a^b.
    \end{equation*}

	Concerning the fitness gain, if there are $R$ replicas of the best individual with fitness $L$, the probability 
	of creating new offspring with fitness $L+1$ is at least
	\begin{equation*}
		\frac{R}{\mu} \cdot (n-L) \cdot p(1-p)^{n-1} \ge \frac{\min{\{\mu(n-L), n\}}}{\mu} \cdot p(1-p)^{n-1},
	\end{equation*}
	therefore we can apply Theorem~\ref{fitness-levels-upper} to estimate this part of the fixed-target time:
	\begin{equation*}
		T_f \le \sum_{i=0}^{k-1}\frac{1}{p_{i}} = \frac{\mu}{p(1-p)^{n-1}}\sum_{i=0}^{k-1}\frac{1}{\min{\{\mu(n-i), n\}}}.
	\end{equation*}

    Unlike~\cite{Witt06}, we consider what the $\min$ clauses can be. Depending on how the target $k$ relates
    to the boundary $b=\floor*{n(1-1/\mu)}$, we write
	\begin{align*}
		T_f \le \begin{cases}
			\frac{\mu}{np(1-p)^{n-1}} \cdot k,                                               &k \le b+1 \\
			\frac{\mu}{np(1-p)^{n-1}} \cdot (b+1) + \frac{H_{n-b-1}-H_{n-k}}{p(1-p)^{n-1}},  &k > b+1
		\end{cases}
	\end{align*}
	which completes the proof, noting that the fixed-target time is $\mu + T_r + T_f$. 
\end{proof}


\section{Drift Analysis}\label{sec:drift}

In this section we consider the drift theorems in the fixed-target context.
These theorems are generally more powerful, but it appears that one should use them for proving fixed-target results
with slightly more care than in the case of fitness levels. We also propose variations of variable and multiplicative drift theorems
that are specifically aimed at gaining precision in the fixed-target context.


\subsection{Drift Theorems}\label{drift:method}
Drift theorems translate bounds on step-wise expected progress into bounds on expected first-hitting runtimes.
They are usually formulated in terms of a random process that needs to hit a certain minimum value. To this end, the search process of the algorithm is mapped to real numbers via so-called potential functions.  
Some of these theorems prohibit the process from falling below the target, or from visiting an interval between a target and the next greater value.
For this reason, some optimization time results cannot be converted into the fixed-target results without additional work,
as targets different from the optimum violate the requirements above.

The paper~\cite{koetzingK-first-hitting-times-thcs19} contains a discussion of processes which may fall below the target, and the implications for drift theorems.
For instance,~\cite[Example 6]{koetzingK-first-hitting-times-thcs19} gives an example of a process with the target $0$ and the expected progress of $1$
at $X_t = 1$, which is given by $X_{t+1} = -n+1$ with probability of $1/n$ and $X_{t+1} = 1$ otherwise.
By mistakenly applying a well-known additive drift theorem from~\cite{HeY01} to this process,
one can get an overly optimistic upper bound of 1 on the expected runtime, which is, in fact, $n$.

We start the discussion with the additive drift theorems. We provide their versions from~\cite{koetzingK-first-hitting-times-thcs19}
which appear to be preadapted to the fixed-target conditions. For the first of these theorems
we explicitly note that its upper bound is not $(X_0 - k) / \delta$, but a (generally) larger value.
Indeed, if we define $X_T$ to be the value of the process at the hitting time $T$,
it is only known that $E[X_T \mid X_0] \le k$, and the latter may be far from being an equality.
Proving the bounds for $E[X_T \mid X_0]$ seems to be the essential additional work in order to prove fixed-target results.

\begin{theorem}[Additive drift, fixed-target upper bounds; Theorem~7 from~\cite{koetzingK-first-hitting-times-thcs19}, original version in~\cite{HeY01}]
Let $k$ be the target value, let $(X_t)_{t \in \N}$ be a sequence of random variables over $\R$, and let $T = \inf\{t \mid X_t \le k\}$.
Suppose that there is some value $\delta > 0$ such that, for all $t < T$, it holds that $X_t - E[X_{t+1} \mid X_0, \ldots, X_t] \ge \delta$.
Then $E[T \mid X_0] \le (X_0 - E[X_T \mid X_0]) / \delta$. \label{additive-drift-upper}
\end{theorem}

\begin{theorem}[Additive drift, fixed-target lower bounds; Theorem~8 from~\cite{koetzingK-first-hitting-times-thcs19}]
Let $k$ be the target value, let $(X_t)_{t \in \N}$ be a sequence of random variables over $\R$, and let $T = \inf\{t \mid X_t \le k\}$.
Suppose that:
\begin{itemize}
    \item there is some value $\delta > 0$ such that, for all $t < T$, it holds that $X_t - E[X_{t+1} \mid X_0, \ldots, X_t] \le \delta$;
    \item there is some value $c \ge 0$ such that, for all $t < T$, it holds that $E[|X_{t+1} - X_t| \mid X_0, \ldots, X_t] \le c$.
\end{itemize}
Then $E[T \mid X_0] \ge (X_0 - E[X_T \mid X_0]) / \delta \ge (X_0 - k) / \delta$. \label{additive-drift-lower}
\end{theorem}

For the lower bound, the simplification to $(X_0 - k) / \delta$ is possible (contrary to the upper bound),
but having a better bound may be desirable. By recalling again~\cite[Example 6]{koetzingK-first-hitting-times-thcs19}, we can see
that taking $E[X_T \mid X_0]$ into account can improve the bound asymptotically. In fixed-target runtime analysis,
such an improvement may occur with very easy targets.

More advanced drift theorems such as the multiplicative drift theorems~\cite{DoerrJW12algo}
and the variable drift theorems~\cite{MitavskiyRC09,Johannsen10,DoerrFW11,lehre-witt-2021} make it easier to prove sharp bounds on hitting times for processes with drift that depends on the current value, a situation that occurs rather frequently in evolutionary computation. 
Most of the popular drift theorems of this sort can be classified using the following properties: they estimate the time for a process to either reach a certain target value $k$
or to surpass a certain threshold value $k'$, and they also may or may not require the process to never fall below the target or to never visit a region between the threshold and
the ultimate termination state (usually zero).

The case analysis of four variants of drift theorems presented in~\cite{koetzingK-first-hitting-times-thcs19} revealed that only two of these four variants are suitable to be used for fixed-target research:
theorems for upper bounds which require to surpass a threshold $k'$, and theorems for lower bounds which require to reach a target $k$.
The former contain an extra addend in their statement (such as ``1+'' or ``$x_{\min}/h(x_{\min})$''), while the latter do not.
This seems to be closely related to the $E[X_T \mid X_0]$ issue in additive drift theorems, which the ``good'' theorems pessimize to the right direction.

We now present or reformulate certain drift theorems for upper bounds, which can be used out of the box for proving the fixed-target results.

\begin{theorem}[Multiplicative drift, upper bounds~\cite{DoerrJW12algo} adapted to fixed-target settings]
Let $k'$ be the threshold value, let $(X_t)_{t \in \N}$ be a sequence of random variables over $\R$, and let $T = \inf\{t \mid X_t < k'\}$.
Furthermore, suppose that:
\begin{itemize}
    \item $X_0 \ge k'$, and for all $t \le T$, it holds that $X_t \ge 0$;
    \item there is some value $\delta > 0$ such that, for all $t < T$, it holds that $X_t - E[X_{t+1} \mid X_0, \ldots, X_t] \ge \delta X_t$.
\end{itemize}
Then $E[T \mid X_0] \le (1 + \ln(X_0 / k')) / \delta$. \label{multiplicative-drift-upper}
\end{theorem}
\begin{proof}
    Apply~\cite{DoerrJW12algo} to the new potential function $X'_t := X_t / k'$. 
\end{proof}

\begin{theorem}[Variable drift, fixed-target upper bounds; Theorem~10 from~\cite{koetzingK-first-hitting-times-thcs19}]
Let $k'$ be the threshold value, let $(X_t)_{t \in \N}$ be a sequence of random variables over $\R$, and let $T = \inf\{t \mid X_t < k'\}$.
Furthermore, suppose that:
\begin{itemize}
   \item $X_0 \ge k'$;  
	\item for all $t \le T$, it holds that $X_t \ge 0$;
   \item there is a monotonically increasing function $h: [k'; \infty) \to \R^{+}$ such that, for all $t < T$, it holds that $X_t - E[X_{t+1} \mid X_0, \ldots, X_t] \ge h(X_t)$.
\end{itemize}
Then $E[T \mid X_0] \le \frac{k'}{h(k')}+\int_{k'}^{X_0} \frac{1}{h(z)} d(z)$. \label{variable-drift-upper}
\end{theorem}

In all these theorems, the conditions are the same as for their corresponding optimization-time variants.
We now present a drop-in replacement version of~\cite[Theorem~7]{DoerrFW11} that enables obtaining lower bounds on fixed-target hitting times
by using the same functions and conditions that were used for proving the lower bounds on the optimization time using that theorem.
Note that, for technical reasons, this version does not support setting the target to the optimum value.

\begin{theorem}[Variable drift, fixed-target lower bounds; adapted from Theorem~7 from~\cite{DoerrFW11}]
Let $k > 0$ be the target value, let $(X_t)_{t \in \N}$ be a sequence of monotonically decreasing random variables over $\R_0^{+}$, and let $T = \inf\{t \mid X_t \le k\}$.
Suppose that there are two continuous monotonically increasing functions $c, h : \R_0^{+} \to \R^{+}$,
and that for all $t < T$ it holds that
\begin{itemize}
    \item $X_{t+1} \ge c(X_{t});$
    \item $E[X_t-X_{t+1} \mid X_t] \le h(c(X_t))$.
\end{itemize}
Then $E[T \mid X_0] \ge \int_{k}^{X_0} \frac{1}{h(z)} d(z)$. \label{variable-drift-lower}
\end{theorem}
\begin{proof}
    We proceed similarly to how it was done in the proof of~\cite[Theorem~7]{DoerrFW11}, however, with a different additive drift theorem in mind.

    Let $g: \R \to \R$ be the function defined by
    \begin{align*}
        g(z) = \begin{cases}
            \frac{z - k}{h(k)} & \text{ if } z < k; \\
            \int_{k}^{z} \frac{1}{h(x)} dx & \text{ if } z \ge k.
        \end{cases}
    \end{align*}
    This function is strictly monotone increasing and continuous on $\R$. Moreover, it is differentiable on $\R$ with 
    \begin{align*}
        g'(z) = \begin{cases}
            \frac{1}{h(k)} & \text{ if } z < k; \\
            \frac{1}{h(z)} & \text{ if } z \ge k.
        \end{cases}
    \end{align*}
    Using $x := X_t$ and $y := X_{t+1}$, the monotonicity of the sequence $(X_t)_{t \in \N}$ and the condition $X_{t+1} \ge c(X_{t})$,
    we get that $c(x) \le y \le x$. According to the mean-value theorem, for all $x \ge k$ 
    there exists $\xi \in (y, x)$ such that
    \begin{align*}
        g'(\xi) = \frac{g(x) - g(y)}{x - y}
    \end{align*}
    which implies
    \begin{align*}
        \frac{g(x) - g(y)}{x - y} \le g'(y) \le g'(c(x)) = \frac{1}{h(c(x))}
    \end{align*}
    since $g'$ decreases monotonically.

    In terms of the function $g$, the hitting time $T$ describes the smallest $t$ with $g(X_t) \le 0$.
    As $g$ is monotone and invertible, it holds for all $t < T$ that
    \begin{align*}
        E[g(X_t) - g(X_{t+1}) \mid g(X_t)] \le E\left[\frac{X_t - X_{t-1}}{h(c(X_t))} \mid X_t\right] \le 1,
    \end{align*}
    where the last inequality comes from the theorem condition $E[X_t-X_{t+1} \mid X_t] \le h(c(X_t))$.
    We apply Theorem~\ref{additive-drift-lower} for $g(X_t)$ and the zero target to complete the proof. 
\end{proof}

\subsection{Overshoot-Aware Drift Theorems}

We have already shown, using additive drift as an example, that in order to obtain good fixed-target bounds,  
an expected overshoot value $E[X_T \mid X_0]$ needs to be considered. However, more advanced theorems cannot easily
make an advantage out of this idea, although it may be desired (especially when analyzing easy targets,
or in the case when the hardest part of the optimization process is not near the optimum,
as it may happen for the $(\mu+1)$ genetic algorithm on certain monotone functions~\cite{lengler-exponential-slowdown-foga19}).
Here, we present variations of the variable drift and multiplicative drift theorems for upper bounds
which explicitly contain the overshoot term $E[X_T \mid X_0]$ in the expression for the hitting time,
for which reason we call them \emph{overshoot-aware}.

\begin{theorem}[Overshoot-aware variable drift, upper bounds]
Let $k'$ be the threshold value, let $(X_t)_{t \in \N}$ be a sequence of random variables over $\R$, and let $T = \inf\{t \mid X_t < k'\}$.
Furthermore, suppose that
\begin{itemize}
    \item $X_0 \ge k'$;
    \item there is a non-decreasing function $h: [k'; \infty) \to \R^{+}$ such that
          for all $t < T$ it holds that $X_t - E[X_{t+1} \mid X_0, \ldots, X_t] \ge h(X_t)$.
\end{itemize}
Then $E[T \mid X_0] \le \frac{k' - E[X_T \mid X_0]}{h(k')} + \int_{k'}^{X_0} \frac{dz}{h(z)}$. \label{variable-ft-upper}
\end{theorem}
\begin{proof}
Let $D = [k'; \infty)$. We define the potential function $g: D \to \R$ by setting 
\begin{equation*}
    g(x) = \int_{k'}^{x} \frac{dz}{h(\max\{z, k'\})}.
\end{equation*}

For any two points $x \ge y$, such that $x, y \in D$ and $x \ge k'$, the following holds:
\begin{equation*}
    g(x) - g(y) = \begin{cases}
        \int_{y}^{x} \frac{dz}{h(z)} \ge \frac{x - y}{h(x)}, & y \ge k' \\
        \frac{k' - y}{h(k')} + \int_{k'}^{x} \frac{dz}{h(z)} \ge \frac{k' - y}{h(x)} + \frac{x - k'}{h(x)} = \frac{x - y}{h(x)}, & y \le k'.
    \end{cases}
\end{equation*}

As a result, the expected potential decrease is:
\begin{align*}
    E[g(X_{t}) - g(X_{t+1}) \mid X_0, \ldots, X_t] \ge \frac{E[X_{t} - X_{t+1} \mid X_0, \ldots, X_t]}{h(X_t)} \ge 1.
\end{align*}

We also note that $g(x)$ is a linear function for $x \in [0; k']$, hence for any
random variable $Z$ taking values from $[0,k']$ it holds that $E[g(Z)] = g(E[Z])$.

Now we apply Theorem~\ref{additive-drift-upper} and prove this theorem:
\begin{align*}
    E[T \mid g(X_0)] &\le g(X_0) - E[g(X_T) \mid g(X_0)] \\ 
					 &= g(X_0) - g(E[X_T \mid X_0]) \\
                     &= g(X_0) - g(k') + g(k') - g(E[X_T \mid X_0]) \\
                     &= \int_{k'}^{X_0} \frac{dz}{h(z)} + \frac{k' - E[X_T \mid X_0]}{h(k')},
\end{align*}
where we use the fact that the random variable $X_T$ takes values from $[0;k']$ together with the observation just above, and $g$ is a bijection, so conditioning on $g(X_0)$ is
equivalent to conditioning on $X_0$. 
\end{proof}

\begin{corollary}[Overshoot-aware multiplicative drift, upper bounds]
Let $D \subseteq \{0\} \cup \R^{+}$ be the set of possible values of the random process
that is defined by the sequence of random variables $(X_t)_{t \in \N}$.
Let $k \in D$ be the target value of the process, let $T = \inf\{ t \mid X_t \le k \}$ be the hitting time,
and let $k' = \inf\{ x \mid x \in D, x > k \}$ be the lower bound on the values from $D$ that are larger than $k$.
Suppose that
\begin{itemize}
    \item $X_0 \ge k'$;
    \item there is some constant value $\delta > 0$ such that,
          for all $t < T$, it holds that $X_t - E[X_{t+1} \mid X_0, \ldots, X_t] \ge \delta X_t$.
\end{itemize}
Then $E[T \mid X_0] \le \frac{1}{\delta} \cdot \left(1 - \frac{E[X_T \mid X_0]}{k'} + \ln \frac{X_0}{k'}\right).$  \label{multiplicative-ft-upper}
\end{corollary}
\begin{proof}
Follows directly from Theorem~\ref{variable-ft-upper} by choosing $h(x) = \delta x$. 
\end{proof}

Note that a proper estimation of the expected overshoot $E[X_T \mid X_0]$ may in fact be tricky, especially given that it is a conditional expectation.
Contrary to Markovian processes, where $E[X_T \mid X_0] = E[X_T \mid X_{T-1}]$ and understanding $X_{T-1}$ may be rather easy, for more involved processes,
such as optimization of multimodal functions or algorithms with self-adjusting parameters,
this value may depend on the whole sequence of values.


\subsection{Applications}
\subsubsection{Minimum Spanning Trees}

We begin with fixed-target bounds for minimum spanning trees solved by the \ea and its variations.
In the context of evolutionary computation, the function to optimize can be defined in different ways.
We follow~\cite{NeumannW07} and use a function which consists of two parts: the number of connected components with a large weight
(to facilitate connecting all the graph vertices), and the weight of the chosen edges. This function is to be minimized.
It is known~\cite{NeumannW07,DoerrJW12algo} that the \ea optimizes this function in expected time $O(m^2(\log{nw_{\max}}))$,
where $m$ is the number of edges, $n$ is the number of vertices, and $w_{\max}$ is the maximum edge weight.

\begin{theorem}
	Starting from a randomly initialized graph, 
	the expected time for Algorithm~\ref{algo:mpl} with $\mu = \lambda = 1$, 
    whose mutation distribution $\mathcal{M}$ selects a single bit to flip with probability $q$,
	to find a graph with at most $k$ connected components is at most $\frac{1}{q} (1+\ln{\frac{m-1}{k}})$.
\end{theorem}
\begin{proof}
	Consider the potential function $g(x)=s-1$, where $s$ is the number of connected components in the subgraph
	which consists of the edges included in the genotype $x$. If there are $s$ such components, 
	there are at least $s-1$ edges, which can be added to decrease the number of components.
	To do that, it is enough to flip at least one bit corresponding to these edges.
	To apply Theorem~\ref{multiplicative-drift-upper}, we estimate the drift as 
	$E[g(X_t)-g(X_{t+1}) \mid g(X_t)=c] \ge cq.$

    The target of $k$ connected components maps to the target potential of $k-1$ and hence to the threshold value $k$. 
    By applying Theorem~\ref{multiplicative-drift-upper} we get the desired bound. 
\end{proof}

\begin{theorem}
	Starting from some spanning tree,
	the expected time for Algorithm~\ref{algo:mpl} with $\mu = \lambda = 1$, 
    whose mutation distribution $\mathcal{M}$ selects exactly two bits to flip with probability $q$,
	to find spanning tree with the weight at most $k$ larger than the minimum possible weight
	is at most $\frac{1}{q} (1+\ln{\frac{(n-1) w_{\max}}{k+1}})$.
\end{theorem}
\begin{proof}
	We again reuse the corresponding result from~\cite{DoerrJW12algo}.
	The process is defined as $X_{t}=w(x)-w_{\text{opt}}$, and~\cite{DoerrJW12algo} gives the drift bound of $E[X_{t}-X_{t+1} \mid X_{t}=x] \ge X_{t} \cdot q$.
	The application of Theorem~\ref{multiplicative-drift-upper} yields the desired upper bound on the fixed-target runtime,
	as $X_0 \le (m-1)w_{\max}$. 
\end{proof}

We give the two-bit probabilities for the common algorithms.
\begin{itemize}
	\item \rls which flips pairs of bits (``2-opt mutation operator''): $q=\frac{2}{m(m-1)}$;
	\item \ea and \mea: $q = p^2(1-p)^{m-2}$;
	\item \rea: $q = \frac{p^2(1-p)^{m-2}}{1-(1-p)^m}$;
	\item \meatwo: $q = p^2(1-p)^{m-2} + (1-p)^m\frac{2}{m(m-1)}$;
	\item \eab: $q = \frac{2}{m(m-1)} \cdot (C_n^{\beta})^{-1} \cdot 2^{-\beta} \cdot \Theta(1)$.
\end{itemize}

Note that \rls which flips pairs of bits is a rather a mind experiment than an algorithm to use,
however, one may use \rls that tosses a coin and flips either single bits or pairs of bits, which just halves
the probability above.


\subsubsection{The \texorpdfstring{\ea}{(1+1) EA} on \om, Lower Bounds}

We prove the lower fixed-target bounds using variable drift.

\begin{theorem}[The lower fixed-target bound on \ea with $p=1/n$ for \om]\label{theoeaom}
	The expected time to find an individual with fitness at least $k$, $2n/3 < k <n$, when
	$n$ is large enough, is at least \label{result:om:lower:drift}
	\begin{align*}
		\left(1-O\left(\frac{n \log n}{n^3}\right)\right)en\left(2\ln{\frac{\sqrt{\frac{n}{3}}+2}{\sqrt{n-k}+2}}+\ln{\frac{n+16(n-k)+32\sqrt{n-k}}{n+16\frac{n}{3}+32\sqrt{\frac{n}{3}}}}\right).
	\end{align*}
\end{theorem}

\begin{proof}
	The basis of this proof is~\cite[Theorem 5]{DoerrFW11}. Our aim is to apply
	Theorem~\ref{variable-drift-lower}, which allows jumps below the target.
	This allows us to use the original the potential function $X_t=n-f(x)$
	and the existing bound on the expected drift~\cite[Lemma 6]{DoerrFW11}:
    $E[X_t-X_{t+1} \mid X_t=s] \le \frac{s}{en}(1+\frac{16s}{n})$.

	Following~\cite{DoerrFW11}, we bound the step size with $c(x)=x-\sqrt{x}$.
	We denote as a bad step the event of increasing the fitness by more than $\sqrt{x}$.
	To condition on that, we estimate the probability of making a bad step for $2 \le x \le n$,
	which was shown in~\cite{DoerrFW11} to be $O(n^{-3})$ for $x \ge 9$. For a good fixed-target result,
    we need to cover the rest. For $5 \le x \le 8$, the probability of a bad step is at most $\binom{8}{3}p^3 = \frac{56}{n^3} = O(n^{-3})$.
    For $2 \le x \le 4$, a similar calculation yields the probability of $O(n^{-2})$.
    Since the latter bound corresponds to only $\Theta(1)$ fitness values, in which the algorithm spends at most $O(n)$ iterations in expectation,
    the union bound over all bad steps during $en \ln n$ iterations is at most $O((n \log n) / n^3)$.
    This is reflected by having the $1 - O((n \log n) / n^3)$ quotient in the result.

    We also reuse the function $h(x)$ from~\cite[Theorem 5]{DoerrFW11}, which is $h(x) \le \frac{x+2\sqrt{x}}{en}(1+\frac{16+32\sqrt{x}}{n})$,
	and apply Theorem~\ref{variable-drift-lower} to get
	\begin{align*}
		&E[T\mid X_{0}] \ge \int_{k'}^{X_0}\frac{en}{(x+2\sqrt{x})(1+\frac{16+32\sqrt{x}}{n})}dx \\ 
		&\ge en \cdot \left[2\ln{\left(\sqrt{x}+2\right)}-\ln{(n+16x+32\sqrt{x})}+\frac{8\arctan{\frac{4\sqrt{x}+4}{\sqrt{n-16}}}}{\sqrt{n-16}}\right]_{k'.}^{X_{0}}
	\end{align*}

	We simplify the expression above by dropping the addend containing the arctangent, since it increases with $x$ and its value is asymptotically smaller than other addends.
	Finally we choose $k' = n - k$ and bound $X_{0}$ from below by $n/3$. The latter choice is due to Chernoff bounds, which show that the algorithm starts with an individual
	with a Hamming distance to the optimum of at least $n/3$ with probability $1 - e^{-\Omega(n)}$, which hides in the leading factor of $1 - O((n \log n) / n^3)$. 
\end{proof}

Figure~\ref{om:profile-new} illustrates that the bound proven in Theorem~\ref{theoeaom} is significantly better than the one from~\cite{LenglerS15}
and captures the essense of the algorithm's behaviour.

\begin{figure}[!t]\centering
\includegraphics[scale=1.22]{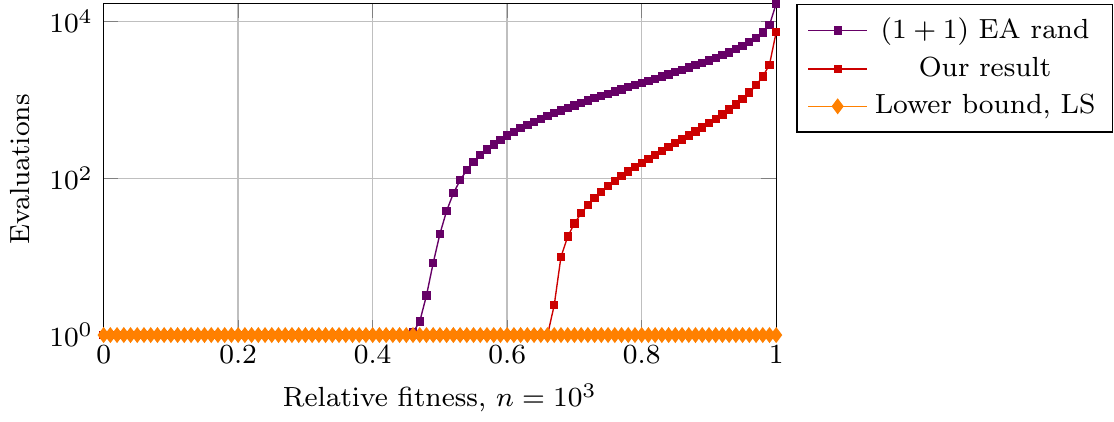}
\par
\includegraphics[scale=1.22]{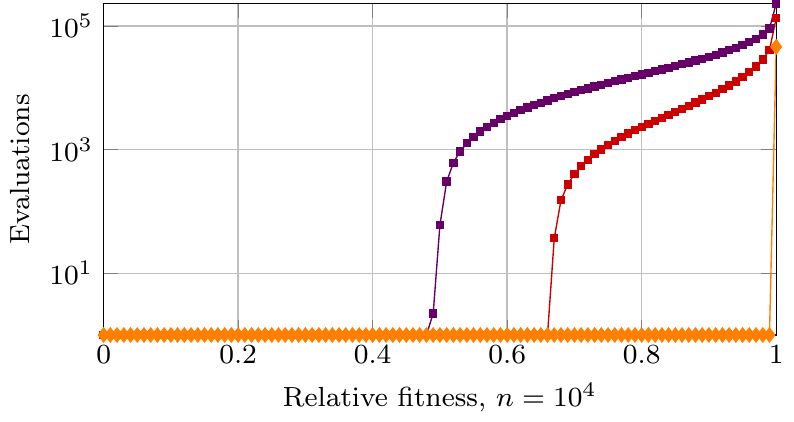}%
\includegraphics[scale=1.22]{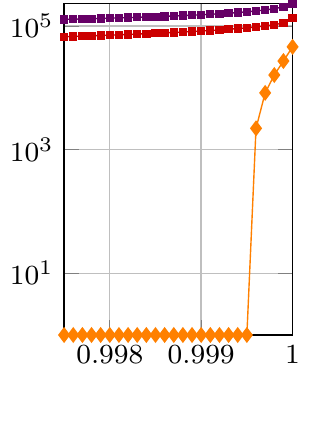}
\par
\includegraphics[scale=1.22]{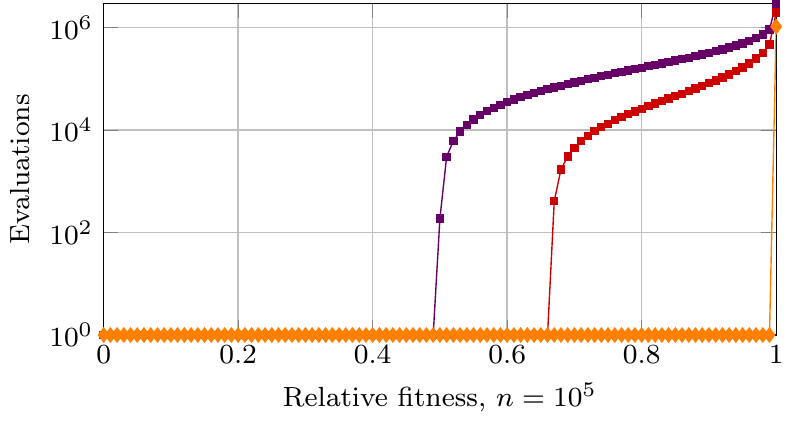}%
\includegraphics[scale=1.22]{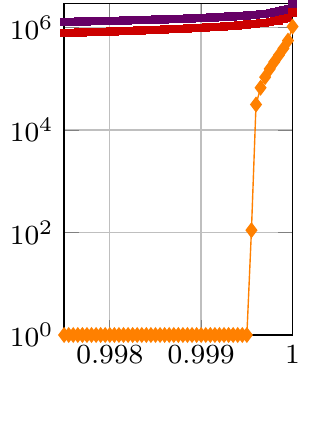}
\par
\caption{The runtime profile of the \ea on \om and the lower bounds, the existing and the new ones.
Plots in the left part of the figure show the general picture for three problem sizes $n \in \{10^3, 10^4, 10^5\}$,
whereas plots on the right zoom into the regions close to the optimum.
The plotted quantity is the expected number of fitness evaluations required to reach the target fitness value against that (relative) fitness value. 
The runtime profiles of the \ea were computed exactly using dynamic programming following the ideas of~\cite{buskulic-doerr}.
For the lower bounds, the respective expressions were evaluated and plotted, disregarding the $(1 \pm o(1))$ factors.
The notation in the legend is as follows:
``\ea rand'' is the \ea starting at a randomly generated bit string,
``Our result'' is the one proven in Theorem~\ref{result:om:lower:drift},
and ``Lower bound, LS'' is the bound from~\cite[Corollary~3]{LenglerS15}.
}\label{om:profile-new}
\end{figure}


\section{Summary of Difficulties of Fixed-Target Analysis}\label{sec:difficulties}

In this work, we have often seen that fixed-target analyses are not much more difficult than classic runtime analyses. However, from the section dedicated to drift theorems we also learned that, in order to obtain fixed-target results from existing optimization time results, more powerful drift theorems sometimes require additional statements to yield bounds that are good enough. More extreme examples are known, such as our earlier attempt with the \bv\ function in~\cite{fixed-target-gecco19}, for which a radically different approach had to be developed.
What is more, sometimes even the easy proofs employing fitness levels required us to dig deeper into the original optimization time result
and refer not to the main theorem formulation, but to the details of its proof.
Below we present an attempt to classify the possible difficulties that may arise when one wants to prove fixed-target results based on the existing results about optimization times.

\begin{enumeratepar}
    \item \emph{Intermediate results can directly be reused, but final theorems are not enough}. Since many of the existing papers aim at proving optimization times,
          they may not expose the intermediate results, which can be used in fixed-target proofs, as separate citable units, such as theorems or lemmas.
          A simple example is the proof of Theorem~\ref{onemax-lower-bounds-levels}, where we could not even point to the equation we used due to it being unnumbered.

          It may even happen that the proof nearly gives the fixed-target result, but it cannot be derived from the theorem statement. The typical reason is that the result gets simplified in a way
          that does not change the optimization time much but affects the possible fixed-target result. Our proof of Theorem~\ref{result:mu+1} diverged from its source
          at the point of simplification of one of the $\min\{\ldots\}$ clauses: if we followed this simplification too, the final result would be much less precise.

          Presenting the results as a set of smaller lemmas and theorems may make it easier to build fixed-target results (and further extensions) atop of them.
          We admit, however, that it requires more work, and some proofs may be too difficult to split into multiple lemmas, which would then have large and cumbersome statements.
    \item \emph{Results can be directly reused, but they are not applicable or sharp enough for the full range of targets}.
          For various reasons the existing optimization time results may be quite sharp on their own, but appear to be too loose in fixed-target contexts.
          A good example of such a situation is a bound which is sharp for hard regions of the search space, but is loose for the remaining part.
          This is the case for all simple upper bounds for \om, which assume that only one-bit flips are beneficial.
          A similar example for lower bounds is~\cite[Corollary~3]{LenglerS15}, which implicitly assumes that optimization in easy parts is performed instantly.
          Such existing results need to be augmented by more work that refines them in parts that are not too meaningful for finding optimization times,
          but essential for good fixed-target results.

          It follows that, surprisingly at first sight, more complicated proofs of optimization times that yield sharper bounds should be easier to adapt to the fixed-target context
          than easier proofs. We admit, however, that such proofs are harder to produce and verify, and the main messages important for understanding optimization times may get obscured.
    \item \emph{Some results can be directly reused, but additional statements need to be proven}.
          One of such cases observed in this paper are drift theorems, which can borrow all the conditions from the existing optimization time results,
          but require to prove the bounds on the overshoot term $E[X_T \mid X_0]$ in order to yield the best possible precision. This is not actually required at times,
          however, on some occasions not doing that results in a ridiculously imprecise bound, such as an upper bound that does not take the target into account.

          This can be seen as a fair price to pay, since drift theorems are applicable for a wider range of processes than fitness level theorems,
          and harder problems may demand more complicated solutions. Based on some of our preliminary work, however, we conjecture that if one has both upper and lower bounds
          that are based on drift theorems and are good enough, the matching fixed-target results may be obtained with much less effort.
    \item \emph{Existing results require more work in order to be usable}.
          In the most complicated cases, the existing optimization time proofs may be too overfitted to the optimum being the target.
          For instance, it may be the case that the distance between the optimum and the target, which used to be zero,
          appears in so many locations of the proof, that the whole proof needs to be significantly reworked and extended to produce the fixed-target results.
          This is what seems be the case for linear functions: some of our work in~\cite{fixed-target-gecco19} for \bv, especially the failed attempts that were not included into that paper, suggests that
          the principles of designing the potential function in~\cite{Witt13} shall be significantly changed to work well in the fixed-target context.
\end{enumeratepar}

While two latter points seem to be fundamental, two former appear to follow from the current habits of performing and presenting the research,
for which reason they may be called the \emph{social} ones: the use of best practices, possibly assisted by automated theorem provers,
has the potential of resolving a large fraction of these difficulties just as a side effect.


\section{Conclusion}

In this first work focussed on fixed-target runtime analysis, most of our results indicate that deriving fixed-target results for the whole set of reasonable targets 
is in several cases not more complicated than just analyzing the classical optimization time (which is the special case where the target is set to the fitness of an optimal solution).
Since fixed-target results are much more informative than the optimization time alone, in such situations we can only advocate to conduct runtime analyses in the more general fixed-target perspective.
As discussed, this often does not need different proofs, all that is required is to formulate the information present in the proof also in the result.
That said, there are also problems which require additional facts to be proven in order to obtain sharp enough fixed-target results. Extending our current toolbox to analyze these situations should be a fruitful direction for further research.

Together with a complementary fine-grained notion that appeared recently, optimization times starting from a good solution~\cite{antipovBD-ppsn20-from-good}, which might also be called
\emph{fixed-start runtime analysis} to unify terms, fixed-target runtime analysis can help to prove runtimes for algorithms to cross fitness intervals.
This has direct applications in designing and understanding hyper-heuristics, as well as in parameter tuning and parameter control.
Since these research areas tend to require especially sharp bounds, it may be the case that the difficulties in applying even the hardest
tools that take care of fine effects, such as target overshooting, may actually pay off well.

\section*{Acknowledgments}
This was supported by a public grant as part of the Investissement d'avenir project, reference ANR-11-LABX-0056-LMH, LabEx LMH,
and by RFBR and CNRS, project number 20-51-15009.


\bibliographystyle{alpha}
\bibliography{alles_ea,ich,extra}

\end{document}